\documentclass[10pt,letterpaper]{article}
\usepackage[margin=1in]{geometry}
\usepackage{times}
\usepackage[authoryear,round]{natbib}
\usepackage[T1]{fontenc}
\usepackage{amsmath,amssymb,amsthm,bm,booktabs,graphicx,xcolor}
\usepackage{url,dsfont,multirow,enumitem}
\usepackage{comment}
\usepackage{placeins,float,longtable}
\usepackage{tikz}
\usetikzlibrary{matrix,arrows.meta,positioning,calc}
\usepackage{hyperref}
\DeclareMathOperator*{\argmax}{arg\,max}
\DeclareMathOperator{\TopK}{TopK}
\DeclareMathOperator{\TV}{TV}
\newcommand{\R}{\mathbb{R}}
\newcommand{\E}{\mathbb{E}}
\newcommand{\ind}{\mathds{1}}
\newcommand{\simplex}{\triangle^n}

\newcommand{\softmax}{\operatorname{softmax}}

\newtheorem{theorem}{Theorem}
\newtheorem{proposition}[theorem]{Proposition}
\graphicspath{{./}}
\hypersetup{hidelinks,pdfauthor={Tom\'as Brogueira, Marcos Treviso, Miguel Couceiro},pdftitle={How Sparse Probability Maps Shape Mixture-of-Experts Routing},pdfsubject={}}

\usepackage{arydshln}
\makeatletter
\def\adl@drawiv#1#2#3{%
        \hskip.5\tabcolsep
        \xleaders#3{#2.5\@tempdimb #1{1}#2.5\@tempdimb}%
                #2\z@ plus1fil minus1fil\relax
        \hskip.5\tabcolsep}
\newcommand{\cdashlinelr}[1]{%
  \noalign{\vskip 2pt
           \global\let\@dashdrawstore\adl@draw
           \global\let\adl@draw\adl@drawiv}
  \cdashline{#1}[.4pt/2pt]
  \noalign{\global\let\adl@draw\@dashdrawstore
           \vskip 2pt}}
\makeatother
\setcitestyle{citesep={;},aysep={,},yysep={;}}
\usepackage{authblk}
\title{\bfseries How Sparse Probability Maps \\ Shape Mixture-of-Experts Routing}
\author[1,2,3]{Tom\'as Brogueira}
\author[1,3,4,5]{Marcos Treviso}
\author[1,2,4]{Miguel Couceiro}
\affil[1]{T\'ecnico, Universidade de Lisboa}
\affil[2]{INESC-ID}
\affil[3]{Instituto de Telecomunica\c{c}\~oes}
\affil[4]{ELLIS Unit Lisbon}
\affil[5]{Gandara AI}
\makeatletter
\ifdefined\AB@maketitle
\let\maketitle\AB@maketitle
\renewcommand{\@author}{
\normalsize
Tom\'as Brogueira$^{1,2,3}$, \quad
Marcos Treviso$^{1,3,4,5}$, \quad
Miguel Couceiro$^{1,2,4}$\\
\textnormal{\small $^1$T\'ecnico, Universidade de Lisboa}\\
\textnormal{\small $^2$INESC-ID \quad $^3$Instituto de Telecomunica\c{c}\~oes}\\
\textnormal{\small $^4$ELLIS Unit Lisbon \quad $^5$Gandara AI}\\
\textnormal{\small Lisbon, Portugal}\\
\textnormal{\small\texttt{tomas.brogueira@tecnico.ulisboa.pt}}
}
\fi
\makeatother
\date{}
\begin{document}
\maketitle
\makeatletter
\ifdefined\AB@maketitle\else
\begin{center}
\small Lisbon, Portugal\\
\texttt{tomas.brogueira@tecnico.ulisboa.pt}
\end{center}
\fi
\makeatother
\begin{abstract}
Mixture-of-experts (MoE) routers typically apply softmax to the router scores and keep the top-$K$ experts, making every token use exactly $K$ experts.
Sparsity-inducing probability maps such as sparsemax, $\alpha$-entmax and normmax can adaptively assign \emph{exact zeros} to selected experts, and therefore appear to offer token-dependent expert participation, even when using the same top-$K$ machinery.
In this work, we study whether and how this sparsity survives training.
We train matched 300M and 1B top-2 MoE language models with softmax, $1.5$-entmax, sparsemax and $2$-normmax, and find that the maps behave very differently once trained: at 1B, entmax discards 30\% less probability mass than softmax while almost never dropping a selected expert, sparsemax retains the most mass, and normmax routes 21\% of tokens to a single expert.
These outcomes are not properties of the maps alone.
Each map drops a selected expert only when the gap between the two largest scores reaches a fixed threshold, and the trained routers differ in the score distribution they learn: the entmax router learns scores with roughly half the spread of softmax's, which keeps its top-2 gaps below its threshold, while sparsemax and normmax, which share the same threshold, learn different gap distributions and hence different participation.
Routers thus co-adapt their scores to the map, and a map's capacity to produce zeros does not by itself determine expert participation.
While none of the sparse maps improves validation loss over softmax, they make the trained models far less sensitive to selecting more experts at inference: sparsemax trained with $K{=}2$ loses $0.02$ nats when run with $K{=}8$, where softmax loses $0.58$.
Our results indicate that adaptive MoE routing has to be designed around the joint behavior of the probability map and the learned scores, rather than around the map alone.\footnote{Our code will be released upon acceptance.}
\end{abstract}

\section{Introduction}
\label{sec:intro}

Mixture-of-experts (MoE) layers scale model capacity by activating only a few expert networks per token~\citep{shazeer2017outrageously,fedus2022switch,jiang2024mixtral}.
A router scores the experts, a probability map turns the scores into weights, and a top-$K$ operator keeps the $K$ highest-weighted experts, whose outputs are combined with the renormalized weights.
With softmax, every token uses exactly $K$ experts, regardless of how many it needs, and the selected experts often account for only a fraction of the router's probability mass---the rest is discarded by truncation.

Sparsity-inducing probability maps promise a way out without changing the top-$K$ machinery.
Sparsemax~\citep{martins2016softmax}, $\alpha$-entmax~\citep{peters2019sparse} and normmax~\citep{santos2024sparse,santos2025hopfield} can assign \emph{exact zeros}, so a selected expert may receive zero weight and drop out of the mixture. 
The number of experts that actually contribute to a token is then \emph{data-dependent}, even when the expert budget $K$ stays fixed.
Several recent routers also build on this idea~\citep{oldfield2024multilinear,thaman2025dual,wang2025remoe}.
However, how sparsity behaves in these models is not obvious.
Since the router scores are learned, changing the map also changes the scores that the map receives: a trained sparsemax router and a trained softmax router differ both in how they transform scores and in which scores they learned to produce.
As a result, a map that can produce zeros may rarely do so once its router has been trained

In this paper, we study whether the sparsity of the probability map survives training, and what determines expert participation when it does not.
We train matched 300M and 1B top-2 MoE language models with softmax, $1.5$-entmax, sparsemax and $2$-normmax under an identical architecture, data and balancing objective, and we measure the probability mass retained by the selected experts (\emph{retention}), the number of selected experts with non-zero weight (\emph{participation}), and predictive quality.
To relate the map to the scores it receives, we record the router scores of every trained model and measure how their spread and top-2 gaps sit relative to each map's sparsification threshold.
Figure~\ref{fig:overview} summarizes the setting and the outcomes.
We find that a map's ability to produce zeros does not, by itself, determine how many experts a token uses.
Our main contributions are:
\begin{itemize}[leftmargin=*]
    \item \textbf{A controlled study of sparse maps as MoE routers.} We train top-2 MoE language models at two scales that differ only in the probability map, and measure retention, participation, load balance and predictive quality, together with the router scores of every trained model (\S\ref{sec:setup}).
    \item \textbf{Retention and participation are distinct outcomes.} Under top-2 selection, mass is discarded only when more than two experts have positive probability, while participation drops only when a single expert does (Equation~\ref{eq:observed-support-decomposition}). The trained routers land in different regimes: at 1B, $1.5$-entmax discards $30\%$ less mass than softmax yet keeps both experts for $99.95\%$ of tokens, whereas normmax-2 routes $21\%$ of tokens to a single expert (\S\ref{sec:survive}).
    \item \textbf{The map's threshold and the learned scores jointly set participation.} Each map drops a selected expert exactly when the top-2 gap reaches a fixed threshold. The sparse routers learn smaller score spreads than softmax, keeping $1.5$-entmax's gaps below its threshold of $2$; sparsemax and normmax-2 share the threshold $1$, yet route $6.9\%$ and $21.3\%$ of tokens to one expert, as their routers learn different gap distributions (\S\ref{sec:map-or-scores}).
    \item \textbf{Sparse maps reduce sensitivity to the inference-time budget.} None improves validation loss at the trained budget, but running the 1B models trained with $K{=}2$ at $K{=}8$ costs sparsemax $0.02$ nats against $0.58$ for softmax, as sparsemax gives zero weight to most of the extra slots (\S\ref{sec:budget}).
\end{itemize}

\begin{figure}[t]
\centering
\definecolor{csoft}{HTML}{4C72B0}
\definecolor{cent}{HTML}{DD8452}
\definecolor{cspa}{HTML}{55A868}
\definecolor{cnor}{HTML}{8172B3}
\providecommand{\barchart}[4]{%
  \draw[black!45, line width=0.4pt] (#1,#2) -- (#1+1.25,#2);
  \foreach \j/\h in {#4}{%
    \pgfmathsetmacro{\xx}{#1+0.05+\j*0.15}%
    \ifnum\j<2 \fill[#3!90!black] (\xx,#2) rectangle (\xx+0.11,#2+\h*1.1);
    \else \fill[black!25] (\xx,#2) rectangle (\xx+0.11,#2+\h*1.1); \fi
  }}
\begin{tikzpicture}[scale=0.9, transform shape,
  font=\scriptsize,
  box/.style={draw=black!55, fill=black!4, rounded corners=2.5pt, minimum width=2.9cm, minimum height=0.5cm,
              text=black!80, line width=0.6pt, inner sep=3pt, align=center},
  mapbox/.style={box, draw=black!75, fill=white, line width=0.9pt},
  arr/.style={-{Stealth[length=4pt,width=3.5pt]}, black!60, line width=0.6pt},
  lab/.style={text=black!70, align=center},
  dk/.style={lab, font=\fontsize{6}{7}\selectfont},
  panel/.style={font=\scriptsize\bfseries, text=black!80},
  every node/.style={align=center}
]
\node[box] (rt) at (1.55,3.4) {router: scores $\bm z_i$};
\node[box] (mp) at (1.55,2.35) {probability map $\bm p_i=\pi(\bm z_i)$\\[-1pt]
  {softmax}\,$\cdot$\,{$\alpha$-entmax}\,$\cdot$\,{normmax-$\gamma$}};
\node[box] (tk) at (1.55,1.3) {top-$2$ selection, renormalize};
\node[box] (ex) at (1.55,0.25) {mixture of the selected experts};
\draw[arr] (rt) -- (mp); \draw[arr] (mp) -- (tk); \draw[arr] (tk) -- (ex);

\node[lab, text=csoft!70!black] at (4.92,3.62) {softmax};
\node[lab, text=cent!70!black]  at (6.77,3.62) {$1.5$-entmax};
\node[lab, text=cspa!70!black]  at (8.62,3.62) {sparsemax};
\node[lab, text=cnor!70!black]  at (10.47,3.62) {normmax-$2$};
\node[lab, anchor=east] at (5.2,3.25) {gap $0.5$};
\barchart{4.3}{2.5}{csoft}{0/0.328,1/0.199,2/0.163,3/0.109,4/0.081,5/0.060,6/0.036,7/0.024}
\barchart{6.15}{2.5}{cent}{0/0.557,1/0.246,2/0.157,3/0.038,4/0.002,5/0,6/0,7/0}
\barchart{8.0}{2.5}{cspa}{0/0.733,1/0.233,2/0.033,3/0,4/0,5/0,6/0,7/0}
\barchart{9.85}{2.5}{cnor}{0/0.602,1/0.266,2/0.132,3/0,4/0,5/0,6/0,7/0}
\node[dk] at (4.92,2.25) {$\delta{=}0.47$, $K_+{=}2$};
\node[dk] at (6.77,2.25) {$\delta{=}0.20$, $K_+{=}2$};
\node[dk] at (8.62,2.25) {$\delta{=}0.03$, $K_+{=}2$};
\node[dk] at (10.47,2.25) {$\delta{=}0.13$, $K_+{=}2$};
\node[lab, anchor=east] at (5.2,1.5) {gap $1.1$};
\barchart{4.3}{0.65}{csoft}{0/0.505,1/0.168,2/0.113,3/0.075,4/0.056,5/0.041,6/0.025,7/0.017}
\barchart{6.15}{0.65}{cent}{0/0.839,1/0.134,2/0.027,3/0,4/0,5/0,6/0,7/0}
\barchart{8.0}{0.65}{cspa}{0/1.0,1/0,2/0,3/0,4/0,5/0,6/0,7/0}
\barchart{9.85}{0.65}{cnor}{0/1.0,1/0,2/0,3/0,4/0,5/0,6/0,7/0}
\draw[cspa!80!black, dashed, line width=0.5pt] (8.2,0.65) rectangle (8.31,0.81);
\draw[cnor!80!black, dashed, line width=0.5pt] (10.05,0.65) rectangle (10.16,0.81);
\node[dk] at (4.92,0.4) {$\delta{=}0.33$, $K_+{=}2$};
\node[dk] at (6.77,0.4) {$\delta{=}0.03$, $K_+{=}2$};
\node[dk] at (8.62,0.4) {$\delta{=}0$, $K_+{=}1$};
\node[dk] at (10.47,0.4) {$\delta{=}0$, $K_+{=}1$};
\fill[black!25] (4.3,-0.2) rectangle (4.45,-0.07);
\node[lab, anchor=west] at (4.47,-0.135) {discarded mass $\delta$};
\draw[black!60, dashed, line width=0.5pt] (6.6,-0.2) rectangle (6.75,-0.07);
\node[lab, anchor=west] at (6.77,-0.135) {filler (zero weight)};
\fill[black!70] (9.0,-0.2) rectangle (9.15,-0.07);
\node[lab, anchor=west] at (9.17,-0.135) {selected pair};

\begin{scope}[yshift=3mm]
\node[lab, anchor=west] at (12.35,3.1) {discarded mass $\bar\delta$};
\foreach \y/\c/\v/\t in {2.78/csoft/0.661/0.66, 2.52/cent/0.462/0.46, 2.26/cspa/0.153/0.15, 2.00/cnor/0.250/0.25}{%
  \fill[\c!90!black] (12.4,\y-0.085) rectangle (12.4+\v*2.2,\y+0.085);
  \node[lab, anchor=west] at (12.4+\v*2.2+0.05,\y) {\t};}
\draw[black!45, line width=0.4pt] (12.4,1.87) -- (12.4,2.91);

\node[lab, anchor=west] at (12.35,1.4) {single-expert tokens};
\foreach \y/\c/\v/\t in {1.08/csoft/0.0/0\%, 0.82/cent/0.002/0.05\%, 0.56/cspa/0.323/6.9\%, 0.30/cnor/1.0/21.3\%}{%
  \fill[\c!90!black] (12.4,\y-0.085) rectangle (12.4+\v*1.8,\y+0.085);
  \node[lab, anchor=west] at (12.4+\v*1.8+0.05,\y) {\t};}
\draw[black!45, line width=0.4pt] (12.4,0.17) -- (12.4,1.21);
\end{scope}
\end{tikzpicture}
\caption{Left: in a top-$2$ MoE router, the probability map $\pi$ is the only component we change. Middle: on the same router scores, the maps differ in how much probability mass falls outside the selected pair (discarded mass $\delta$, gray) and in whether a selected expert receives zero weight (a \emph{filler}, dashed), which lowers the number of contributing experts $K_+$; shown for two illustrative $8$-expert score vectors whose top-2 gaps are $0.5$ and $1.1$. Right: statistics of our trained 1B MoE models.}
\label{fig:overview}
\end{figure}
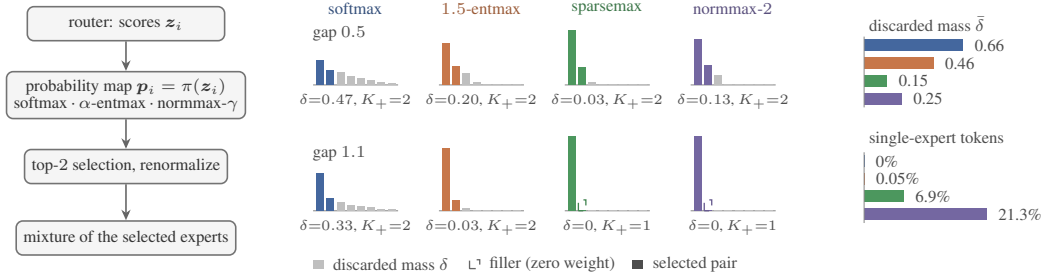

\section{Background: probability maps and expert selection}
\label{sec:background}

A bias-free router maps a token representation $\bm x_i\in\R^d$ to scores $\bm z_i=\bm W_r\bm x_i\in\R^n$ and to probabilities $\bm p_i=\pi(\bm z_i)\in\simplex$, where $\simplex=\{\bm p\in\R^n:\bm p\geq\bm0,\ \bm1^\top\bm p=1\}$ is the $(n-1)$-dimensional probability simplex.
We write $\ind[\cdot]$ for the indicator function, and $[a]_+=\max(a,0)$.

\paragraph{Selection and combination.}
For $1\leq K\leq n$, let $\mathcal S_i=\TopK(\bm p_i,K)$ be the $K$ selected experts (ties are broken by score, then by index).
The \emph{contributing} experts are those selected with positive weight, $\mathcal T_i=\{j\in\mathcal S_i:p_{ij}>0\}$.
The retained mass, the normalized mixture weights and the layer output are
\begin{equation}
m_i=\sum_{j\in\mathcal S_i}p_{ij},\qquad
q_{ij}=\frac{p_{ij}\,\ind[j\in\mathcal S_i]}{m_i},\qquad
\operatorname{MoE}(\bm x_i)=\sum_{j\in\mathcal S_i}q_{ij}\,E_j(\bm x_i),
\label{eq:combine}
\end{equation}
where $E_j$ is the $j$th expert network.
Under softmax, $\mathcal T_i=\mathcal S_i$ always.
Under a sparse map, a selected expert can have $p_{ij}=0$; we call such a selection a \emph{filler}.
The task loss depends on the router only through $\bm q_i$, while the full distribution $\bm p_i$ also enters the balancing objective below.

\subsection{Sparsity-inducing probability maps}
\label{sec:maps}
We suppress the token index and write $\bm p=\pi(\bm z)$.
All maps we consider are regularized argmax transformations~\citep{blondel2020learning},
\begin{equation}
\pi_H(\bm z)=\argmax_{\bm p\in\simplex}\ \bm p^\top\bm z+H(\bm p),
\label{eq:regularized}
\end{equation}
and differ only in the entropy $H$.
Shannon entropy, $H^{\mathrm S}(\bm p)=-\sum_jp_j\log p_j$, gives softmax, $p_j=e^{z_j}/\sum_ke^{z_k}$, which is dense: every expert receives positive probability at finite scores.
The Tsallis entropy $H^{\mathrm{Ts}}_\alpha(\bm p)=\big(1-\sum_jp_j^\alpha\big)/(\alpha(\alpha-1))$ with $\alpha>1$ gives $\alpha$-entmax~\citep{peters2019sparse},
\begin{equation}
p_j=\big[(\alpha-1)z_j-\tau\big]_+^{1/(\alpha-1)},
\label{eq:entmax}
\end{equation}
where the threshold $\tau$ is the unique value for which the probabilities sum to one.
Experts whose score falls below $\tau/(\alpha-1)$ receive exactly zero probability; $\alpha\to1$ recovers softmax and $\alpha=2$ gives sparsemax, the Euclidean projection onto the simplex~\citep{martins2016softmax}.
The norm entropy $H^{\mathrm N}_\gamma(\bm p)=1-\|\bm p\|_\gamma$ with $\gamma>1$ gives $\gamma$-normmax~\citep{santos2024sparse,santos2025hopfield},
\begin{equation}
p_j=\frac{[z_j-\tau]_+^{1/(\gamma-1)}}{\sum_k[z_k-\tau]_+^{1/(\gamma-1)}},
\qquad \sum_j[z_j-\tau]_+^{\gamma/(\gamma-1)}=1,
\label{eq:normmax}
\end{equation}
which is also sparse but, unlike entmax, favors distributions that are closer to uniform on their support.
We use $\alpha=1.5$, for which the threshold has an exact sorting-based solution and which has been the default in sparse attention work~\citep{peters2019sparse,correia-etal-2019-adaptively}, and $\gamma=2$, the Euclidean case studied by \citet{santos2025hopfield}.
All four maps are differentiable almost everywhere and are trained by ordinary backpropagation.

\paragraph{Singleton thresholds.}
With the normalizations above, a map assigns the whole mass to a single expert exactly when the gap $g=z_{(1)}-z_{(2)}$ between the two largest scores reaches a fixed threshold, namely $g\geq 2$ for $1.5$-entmax and $g\geq 1$ for both sparsemax and normmax-2 (Appendix~\ref{app:thresholds}); softmax has no finite threshold.
In particular, on the same scores, sparsemax and normmax-2 always agree on whether a token has one or two contributors.

\paragraph{Balancing objective.}
For a microbatch of $M$ tokens, we use the Switch-style objective~\citep{fedus2022switch},
\begin{equation}
f_j=\frac1{MK}\sum_i\ind[j\in\mathcal S_i],\qquad
\bar p_j=\frac1M\sum_i p_{ij},\qquad
\mathcal L_{\mathrm{LB}}=n\sum_jf_j\bar p_j,
\label{eq:loads}
\end{equation}
where the selection frequencies $\bm f$ include fillers and are detached from differentiation.
Because $\mathcal L_{\mathrm{LB}}$ uses the full probabilities $\bm p_i$ rather than the selected mixture $\bm q_i$, its gradient depends on the map even at a shared coefficient; under a sparse map, an expert with $p_{ij}=0$ receives no balancing gradient from token $i$.

\subsection{Routing measurements}
\label{sec:metrics}
We characterize routing by how much probability mass survives selection, how many selected experts contribute, and how evenly assignments are spread over experts.
The \emph{discarded mass} of token $i$ is the mass of the unselected experts,
\begin{equation}
\delta_i=\sum_{j\notin\mathcal S_i}p_{ij}=1-m_i\in[0,\,1-K/n],
\label{eq:delta}
\end{equation}
and we report its average $\bar\delta$ over tokens.
Since renormalization moves exactly this mass onto the selected experts, $\delta_i$ is also the total variation distance between the routing distribution and its selected mixture, $\TV(\bm p_i,\bm q_i)=\tfrac12\|\bm p_i-\bm q_i\|_1=\delta_i$.
The \emph{number of contributors} is $K_{+,i}=|\mathcal T_i|\leq K$. 
Under top-2 selection, $K_{+,i} = 1$ whenever a selected expert has zero weight.
Finally, we compare two assignment histograms over experts: the \emph{nominal} load counts every selection, fillers included, and the \emph{contributor} load counts only positive-weight selections, each assignment adding $1/K$.
We report their coefficient of variation (CV) across experts~\citep{shazeer2017outrageously}, together with the standard deviation, since removing fillers also changes the mean.
Appendix~\ref{app:metrics} gives the aggregation details.

\section{Experimental setup}
\label{sec:setup}

\paragraph{Models and training.}
We train MoE language models at two scales with a Llama~3-style backbone~\citep{dubey2024llama3} in which every second feed-forward block is replaced by an MoE layer: a 300M model (8 layers, width 320, 4 MoE layers with 26 experts each) and a 1B model (9 layers, width 768, 4 MoE layers with 27 experts each).
Every token selects $K=2$ experts, following Mixtral's selection and weighted combination~\citep{jiang2024mixtral}, without capacity constraints or token dropping; the 1B model has $302$M active parameters.
The four maps share everything except $\pi$, including the balancing objective of Equation~\eqref{eq:loads} with coefficient $0.01$, averaged over MoE layers and evaluated on local microbatches.
A dense model of matching depth and width at each scale provides a reference with a larger active-parameter budget.
All models are trained on DCLM-baseline~\citep{li2024datacomp} for $10$B tokens ($19{,}073$ updates of $524{,}288$ tokens) with AdamW, cosine decay and bfloat16 mixed precision. Appendices~\ref{app:architecture} and~\ref{app:training} give the architecture details and the full training and evaluation recipe.

\paragraph{Evaluation.}
For the original matched comparison, we evaluate the final checkpoint of every run, which is also the one with the lowest DCLM validation loss.
Predictive quality is measured by validation loss, mean accuracy over eight OLMES tasks and LAMBADA, and LAMBADA word perplexity~\citep{gu2025olmes,paperno2016lambada}.
For each OLMES task and model, we use the higher accuracy of the multiple-choice and completion formulations, following the benchmark's recommendation~\citep[\S3.4]{gu2025olmes}.

Routing diagnostics use $61$ documents from the WikiText-103 test subset of Paloma~\citep{magnusson2024paloma} at context $2{,}048$, giving $285{,}107$ targets per MoE layer. 
All routing statistics are averaged with equal weight over the four MoE layers.
For every trained model we also record the router scores of all four layers on these documents, from which we compute the score spreads and top-2 gap distributions of \S\ref{sec:map-or-scores}.

\section{Results}
\label{sec:results}

We first check that the MoEs are competent models (\S\ref{sec:quality}), then ask whether the sparsity of the map survives training (\S\ref{sec:survive}), whether the differences between trained routers come from the map or from the scores learned under it (\S\ref{sec:map-or-scores}), and what the sparse maps buy at inference time (\S\ref{sec:budget}).
We report both scales and focus the discussion on the 1B models. The 300M models show the same qualitative routing patterns, with their training curves and load statistics in Appendix~\ref{app:small-scale}.

\subsection{Model quality}
\label{sec:quality}
\label{sec:quality-and-k}

\begin{table}[t]
\caption{Predictive quality of the 300M and 1B models: DCLM validation loss, perplexity and accuracies (\%) on LAMBADA (LMB) and the eight OLMES tasks (ARC-Easy, ARC-Challenge, CommonsenseQA, HellaSwag, OpenBookQA, PIQA, SocialIQA, WinoGrande).
The dense references have a larger active-parameter budget. 
}
\label{tab:quality-primary}
\label{tab:quality-primary-detailed}
\centering\small
\setlength{\tabcolsep}{2.4pt}
\begin{tabular}{lcccccccccccc}
\toprule
\bf Model & \bf Val. & \bf LMB (ppl) & \bf LMB & \bf ARC-E & \bf ARC-C & \bf CSQA & \bf HS & \bf OBQA & \bf PIQA & \bf SQA & \bf WG & \bf Avg. \\
\midrule
\textcolor{gray}{\textit{300M params.}} \\
Dense        & 3.411 & 44.8 & 31.8 & 46.7 & 28.6 & 37.2 & 30.0 & 31.8 & 60.3 & 42.3 & 50.0 & 39.9 \\
\cdashlinelr{1-13}
Softmax      & \bf 3.534 & 104.2 & 24.3 & 42.7 & 27.6 & 36.5 & 29.4 & 29.6 & 61.9 & 41.4 & 50.8 & \bf 38.2 \\
$1.5$-entmax & 3.536 & \bf 102.8 & 24.2 & 43.0 & 27.1 & 34.6 & 29.4 & 27.8 & 60.0 & 41.7 & 51.7 & 37.7 \\
Sparsemax    & 3.538 & 104.0 & 24.6 & 42.7 & 26.5 & 36.0 & 27.8 & 28.4 & 59.8 & 40.3 & 52.5 & 37.6 \\
Normmax-$2$  & 3.536 & 105.9 & 23.4 & 42.8 & 25.5 & 33.3 & 29.6 & 25.2 & 60.5 & 40.9 & 50.5 & 36.9 \\
\midrule
\textcolor{gray}{\textit{1B params.}} \\
Dense        & 3.103 & 18.8 & 41.6 & 56.0 & 29.3 & 47.3 & 39.0 & 33.8 & 63.9 & 42.7 & 51.3 & 45.0 \\
\cdashlinelr{1-13}
Softmax      & \bf 3.232 & \bf 28.9 & 36.3 & 50.1 & 28.6 & 43.7 & 34.2 & 33.8 & 64.1 & 42.9 & 51.5 & \bf 42.8 \\
$1.5$-entmax & 3.240 & 29.5 & 35.4 & 50.3 & 28.3 & 42.1 & 34.6 & 30.4 & 63.5 & 41.9 & 51.3 & 42.0 \\
Sparsemax    & 3.242 & 30.3 & 35.1 & 50.1 & 28.3 & 41.9 & 32.6 & 30.8 & 63.7 & 43.2 & 52.1 & 42.0 \\
Normmax-$2$  & 3.244 & 32.1 & 34.8 & 51.8 & 27.5 & 41.2 & 34.9 & 30.6 & 64.8 & 41.6 & 51.9 & 42.1 \\
\bottomrule
\end{tabular}
\end{table}

Table~\ref{tab:quality-primary} reports predictive quality.
All four MoEs are within $0.14$ nats of validation loss and $3$ points of mean accuracy of the dense reference, which has $3.3\times$ their active parameters, so the routers we study are those of reasonable language models rather than of a toy setting.
Among the MoEs, while softmax has the lowest validation loss and the highest average numbers, we note that the margins are small across all maps: $0.008$--$0.012$ in validation loss and $0.7$--$0.8$ in average points. 

Moreover, individual task leaders vary (normmax-2 leads ARC-Easy and PIQA, sparsemax leads SocialIQA and WinoGrande). 
These runs show no consistent quality advantage for the sparse maps. The rest of this section asks what they change in routing.

\subsection{Does the sparsity of the map survive training?}
\label{sec:survive}
\label{sec:routing-results}

\begin{figure}[t]
\centering
\includegraphics[width=0.9\linewidth]{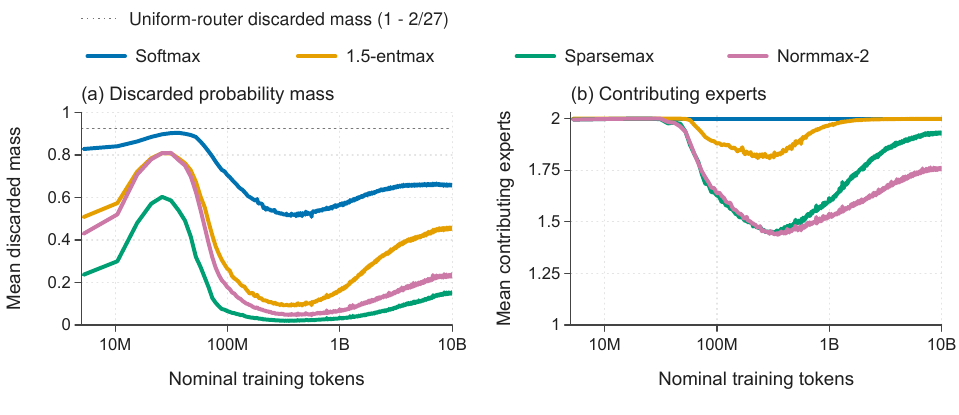}
\vspace*{-0.5cm}
\caption{Discarded probability mass (a) and number of contributing experts per token (b) during training of the 1B MoE models, measured on training batches. 
The horizontal axis is the number of training tokens (log scale). 
All maps sparsify early in training; entmax then returns to two contributors while sparsemax and normmax do not. 
The 300M curves are provided in Appendix~\ref{app:small-scale}. 
}
\label{fig:discarded}
\end{figure}

\begin{table}[t]
\caption{Routing statistics of the trained models.
Besides discarded mass $\bar\delta$ and activated experts $\overline K_+$, we also report the fraction of token--layer states with singleton support and with support of exactly two experts, the discarded mass conditional on support larger than two, the coefficient of variation of the expert loads counting every selection or only positive-weight selections (Contrib. CV), and the standard deviation of the contributor loads (Contrib. SD). Lower CV and SD indicate more even expert loads. Bold marks the lowest value at each scale, including ties at the displayed precision. 
Rows under softmax-based load balancing means that we used softmax probabilities for the auxiliary balancing loss.}
\label{tab:headline}
\centering\small
\begin{tabular}{l rrrrcccc}
\toprule
\bf Map & $\bar\delta$ & $\overline K_+$ & $s{=}1$ \% & $s{=}2$ \% & $\E[\delta\mid s{\geq}3]$ & \shortstack{Nominal\\CV $\downarrow$} & \shortstack{Contrib.\\CV $\downarrow$} & \shortstack{Contrib.\\SD $\downarrow$} \\
\midrule
\textcolor{gray}{\textit{300M params.}} \\
Softmax      & 0.566 & 2.000 & 0.000  & 0.000  & 0.566 & 0.404 & 0.404 & 0.016 \\
$1.5$-entmax & 0.311 & 1.997 & 0.295  & 0.388  & 0.314 & 0.400 & \textbf{0.400} & 0.015 \\
Sparsemax    & 0.099 & 1.904 & 9.598  & 28.174 & 0.160 & 0.434 & 0.450 & 0.017 \\
Normmax-$2$  & 0.146 & 1.769 & 23.144 & 16.228 & 0.241 & \textbf{0.367} & 0.404 & \textbf{0.014} \\
\cdashlinelr{1-9}
\textcolor{gray}{\textit{1B params.}} \\
Softmax      & 0.661 & 2.000 & 0.000  & 0.000  & 0.661 & 0.339 & 0.339 & \textbf{0.013} \\
$1.5$-entmax & 0.462 & 2.000 & 0.048  & 0.095  & 0.462 & \textbf{0.338} & \textbf{0.338} & \textbf{0.013} \\
Sparsemax    & 0.153 & 1.931 & 6.883  & 14.079 & 0.194 & 0.414 & 0.424 & 0.015 \\
Normmax-$2$  & 0.250 & 1.787 & 21.256 & 7.587  & 0.351 & 0.341 & 0.388 & \textbf{0.013} \\
\midrule
\multicolumn{9}{l}{\textcolor{gray}{\textit{300M params. with softmax-based load balancing}}} \\
$1.5$-entmax & 0.292 & 1.999 & 0.079 & 0.245 & 0.293 & 0.485 & 0.485 & 0.019 \\
Sparsemax & 0.078 & 1.921 & 7.941 & 30.019 & 0.126 & 0.643 & 0.668 & 0.025 \\
\bottomrule
\end{tabular}
\end{table}

Figure~\ref{fig:discarded} shows the two routing quantities during training.
All three sparse maps drop contributors early on, but entmax recovers to almost exactly two contributors per token, whereas sparsemax and normmax settle at $1.93$ and $1.79$.
At the end of training, reported in Table~\ref{tab:headline}, $1.5$-entmax discards $30\%$ less probability mass than softmax ($\bar\delta=0.462$ versus $0.661$) with a contributor count that is still $2.000$ to three decimals, sparsemax retains the most mass ($\bar\delta=0.153$), and normmax uses the fewest contributors ($\overline K_+=1.787$). 
The two quantities are governed by different regimes of the support size $s$, the number of experts with positive probability.
Under top-2 selection,
\begin{equation}
 \E[\delta]=\Pr(s\geq3)\;\E[\delta\mid s\geq3],
 \qquad \E[K_+]=2-\Pr(s=1),
 \label{eq:observed-support-decomposition}
\end{equation}
since only a support larger than $2$ can lose mass to truncation, and only a singleton support can reduce participation below $2$.
Entmax sits almost entirely in the first regime: $99.9\%$ of its states have $s\geq3$, so it retains more mass than softmax while keeping both experts.
Normmax sits in the second: $21.3\%$ of its states are singletons, versus $6.9\%$ for sparsemax, which instead has more states with exactly two positive probabilities ($14.1\%$ versus $7.6\%$), retaining all mass with two contributors.
Sparsity in the routing distribution and sparsity in participation are therefore different outcomes, and a map can deliver one without the other.
Results for the 300M models, provided in Appendix~\ref{app:small-scale}, share the same findings.

\subsection{Is it the map, or the scores learned under it?}
\label{sec:map-or-scores}

The trained models differ both in the map and in the scores that the map receives, and the two are easy to separate because each map has a fixed sparsification threshold (see \S\ref{sec:maps}). Concretely, a selected expert is dropped exactly when the top-2 score gap is at least $2$ for $1.5$-entmax and at least $1$ for sparsemax and normmax-2.
The singleton rate of a trained model is therefore nothing but the fraction of its learned gaps that reach its map's threshold.
Figure~\ref{fig:score-gaps} shows these gap distributions for the 1B models, with both thresholds marked in every panel.
The numerical spreads and quantiles at both scales are retained in Table~\ref{tab:score-gaps} in Appendix~\ref{app:score-distributions}.

\begin{figure}[t]
\centering
\includegraphics[width=\linewidth]{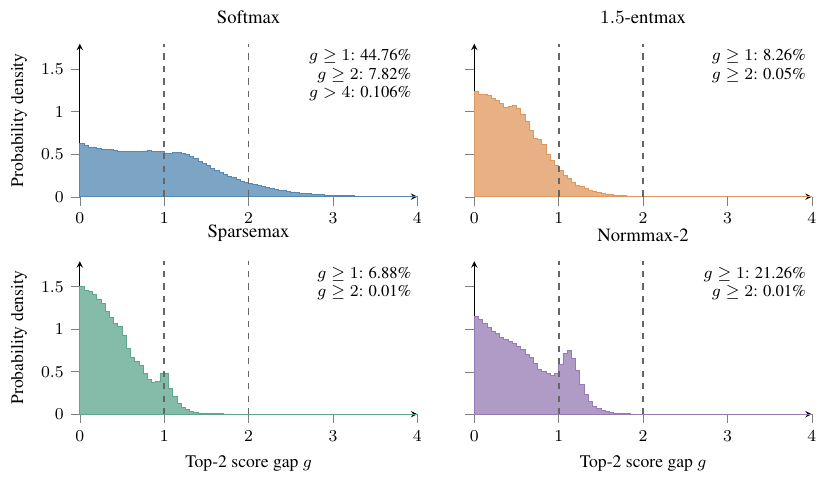}
\caption{Top-2 score gaps of the trained 1B routers. 
Each panel pools $285{,}107$ targets from each of the four MoE layers with equal weight.
Gray dashed lines mark the singleton thresholds $g=1$ for sparsemax and normmax-2 and $g=2$ for $1.5$-entmax. 
Softmax has no finite singleton threshold. 
Panel annotations give the fractions at or above each threshold, computed on the full sample, including density beyond $g=4$ for softmax.
The peaks near $g=1$ for sparsemax and normmax-2 are consistent with gradient saturation: for $g>1$, these maps are locally constant and gradients through the routing probabilities vanish.
}
\label{fig:score-gaps}
\end{figure}

The learned scores turn out to matter as much as the map.
All three sparse routers learned markedly smaller score spreads than the softmax router, with an RMS between $0.32$ and $0.38$ against $0.86$. Under a sparse map, this scale determines how often a score gap reaches the singleton threshold.
For $1.5$-entmax the consequence is decisive.
Its $99$th-percentile gap is $1.45$, well below the threshold of $2$, so the trained router produces a singleton in only $0.05\%$ of states, even though $7.8\%$ of the softmax router's gaps reach that threshold.
Sparsemax and normmax-2 tell a complementary story.
The two maps share the threshold $1$, so on any given score vector they always agree on whether a token keeps one or two contributors, and yet the trained models differ by $0.14$ contributors per token.
The difference can only come from the scores the two routers learned, and it does: the median gap of the normmax-2 router is $0.61$ against $0.40$ for the sparsemax router, and its gaps reach $1$ three times as often.

Taken together, $1.5$-entmax keeps both experts because its threshold is the highest of the three maps and because its router learned scores that stay under it, while sparsemax and normmax-2 differ only in what their routers learned.
Participation is thus set jointly by the map's threshold and by the learned score distribution, and the router adapts the latter to the former.

\paragraph{Depth.}
The layer averages hide a strong depth dependence. In Table~\ref{tab:layerwise} in the appendix, we note that normmax-2 routes $54\%$ of tokens to a single expert in the first MoE layer but only $7\%$ in the last. 
Sparsemax shows the same pattern with smaller magnitude ($19\%$ to $4\%$), and $1.5$-entmax has essentially no singletons anywhere.
Early layers learn the sparsest routers, which suggests that the useful degree of participation is itself depth-dependent and provides a further reason not to set $K$ arbitrarily.\looseness=-1

\subsection{What the sparse maps buy at inference time}
\label{sec:budget}

\begin{table}[t]
\caption{Robustness to the inference-time budget. NLL (nats/target) on the retained DCLM diagnostic sample when a model trained with $K{=}2$ is run, with frozen weights, with $K\in\{1,2,4,8\}$ selected experts; $\Delta$ is the increase from $K{=}2$ to $K{=}8$. Retention and contributor counts for every setting are in Appendix~\ref{app:frozen-k}.}
\label{tab:frozen-k}
\centering\small
\setlength{\tabcolsep}{4pt}
\begin{tabular}{lccccc ccccc}
\toprule
& \multicolumn{5}{c}{300M} & \multicolumn{5}{c}{1B} \\
\cmidrule(lr){2-6} \cmidrule(lr){7-11}
\bf Map & $K{=}1$ & $K{=}2$ & $K{=}4$ & $K{=}8$ & $\Delta$ & $K{=}1$ & $K{=}2$ & $K{=}4$ & $K{=}8$ & $\Delta$ \\
\midrule
Softmax      & 3.703 & 3.560 & 3.623 & 3.913 & $+0.354$ & 3.360 & 3.254 & 3.328 & 3.835 & $+0.582$ \\
$1.5$-entmax & 3.711 & 3.562 & 3.598 & 3.693 & $+0.131$ & 3.363 & 3.261 & 3.306 & 3.512 & $+0.250$ \\
Sparsemax    & 3.715 & 3.565 & 3.566 & 3.568 & $\mathbf{+0.004}$ & 3.365 & 3.263 & 3.274 & 3.284 & $\mathbf{+0.021}$ \\
Normmax-$2$  & 3.703 & 3.562 & 3.574 & 3.586 & $+0.024$ & 3.368 & 3.266 & 3.286 & 3.329 & $+0.064$ \\
\bottomrule
\end{tabular}
\end{table}

Although they do not improve quality, the sparse maps make the trained model much less sensitive to selecting more experts at inference.
We carry that analysis in Table~\ref{tab:frozen-k}.
When the models trained with $K{=}2$ are run with $K{=}8$ and frozen weights, softmax loses $0.58$ nats, because it is forced to spread weight over six experts it never learned to combine.
Sparsemax loses only $0.02$ nats, since it keeps assigning zero weight to most of the extra slots and uses $4.3$ experts on average at $K{=}8$, and normmax-2 and $1.5$-entmax fall in between.
The robustness is one-sided: reducing the budget to $K{=}1$ costs every map about $0.10$ nats.
For every map the loss is lowest at the trained budget, so selecting more experts to retain more mass does not help either, which confirms that retention, participation, and quality are three distinct quantities.

\paragraph{Load balance.}
\label{sec:balance}
Counting contributors rather than nominal selections also changes the picture of expert balance, as the last three columns of Table~\ref{tab:headline} show.
The $1.5$-entmax router has the lowest load CV under both conventions ($0.338$), and because almost all of its selections carry positive weight, excluding the fillers changes nothing for it.
Normmax-2 is different: its nominal CV is only slightly above softmax's ($0.341$ against $0.339$), but once the fillers are excluded its contributor CV rises to $0.388$, so its load is less balanced than the nominal histogram suggests, and sparsemax is the least balanced under either convention.
The balancing objective is one possible source of these differences.
Because it uses the full distribution, an expert with zero probability receives no direct balancing gradient under a sparse map, and changing the score scale also changes the probabilities seen by this objective.

To examine the balancing mechanism, we trained 300M $1.5$-entmax and sparsemax routers with $\softmax(\bm z)$ used only in the auxiliary loss. 
Routing and expert combination still use their respective sparse maps. For $1.5$-entmax, nominal and contributor CV both rise from $0.400$ to $0.485$, while centered score RMS rises from $0.517$ to $0.603$ and the singleton fraction falls from $0.295\%$ to $0.079\%$. Sparsemax shows the same pattern. Nominal CV rises from $0.434$ to $0.643$ and contributor CV from $0.450$ to $0.668$ (Table~\ref{tab:headline}), while centered score RMS rises from $0.413$ to $0.442$ and the singleton fraction falls from $9.598\%$ to $7.941\%$. Since higher CV means a less even distribution of traffic across experts, the dense auxiliary probability does not improve load balance in either run. Figure~\ref{fig:score-gaps-auxsoftmax} shows why fewer tokens become singletons: the entmax tail near $g=2$ thins, and the sparsemax peak just above $g=1$ becomes smaller. More tokens therefore keep both selected experts, but their contributions are distributed less evenly across the expert pool. Per-token participation and load balance move in opposite directions in both runs, which does not support the simple explanation that sparsemax's imbalance comes solely from zero-probability experts receiving no direct balancing gradient.\looseness=-1

\section{Related work}
\label{sec:related}

\paragraph{Sparse maps and differentiable gates in MoE routers.}
Softmax followed by top-$K$ remains the dominant router~\citep{shazeer2017outrageously,fedus2022switch,jiang2024mixtral}, with sigmoid scores renormalized among the selected experts as the main production alternative~\citep{deepseekv3}.
Sparse probability maps have been used as routers before. For example, Multilinear MoE compares entmax with softmax over factorized experts~\citep{oldfield2024multilinear}, DUAL studies sparsemax under routing constraints~\citep{thaman2025dual}, and ReMoE replaces the map by rectified gates with a separate sparsity controller~\citep{wang2025remoe}.
Differentiable relaxations of the selection step itself include DSelect-k, a smooth gate with explicit cardinality control~\citep{hazimeh2021dselect}, the sparse top-$k$ operators of \citet{sander2023fast}, Soft MoE, which replaces discrete routing by soft slot assignment~\citep{puigcerver2023sparse}, and, more recently, SoftMoE, a truncated LapSum relaxation with learned per-layer budgets~\citep{zasada2026softmoe}, and ProbMoE, which samples cardinality-constrained subsets and differentiates through inclusion marginals~\citep{zhao2026probmoe}.
These works propose new routers and evaluate them end-to-end, whereas we keep the router fixed to top-2 and ask what a sparse map delivers once the router has been trained with it, and how the score distribution it learns interacts with the map's sparsification threshold (\S\ref{sec:map-or-scores}).

\paragraph{Load balancing and routing geometry.}
Balancing objectives range from the Switch-style loss we use~\citep{fedus2022switch} to balanced assignment~\citep{lewis2021base,zhou2022expertchoice}, loss-free bias adjustment~\citep{wang2024lossfree} and $\phi$-Balancing, which targets population-level balance through a convex potential and mirror descent~\citep{ugPhi2026}.
\citet{ugYoon2026} show that the task gradient reaches only the selected experts and that balancing provides no token-level counterfactual signal, and \citet{tranhuu2026discontinuities} analyze the switching surfaces of top-$K$ routing and propose smoothing near them.
Our observation that the balancing objective interacts with the map through the full distribution (\S\ref{sec:budget}) is complementary: it concerns which experts the balancing gradient can reach under a sparse map.

\paragraph{Sparse probability transformations.}
Sparsemax~\citep{martins2016softmax} and $\alpha$-entmax~\citep{peters2019sparse} are regularized argmax maps in the Fenchel--Young framework of \citet{blondel2020learning}, which also covers the norm entropies behind normmax~\citep{santos2024sparse,santos2025hopfield}.
Their sparsity has been exploited in attention, where the support size is known to depend on the scale of the scores: \citet{correia-etal-2019-adaptively} learn $\alpha$ per head, and \citet{vasylenko2026longcontext} learn a temperature so that $\alpha$-entmax attention interpolates between sparse and dense regimes as the context grows.
Our results transfer this observation to routing, where the score scale learned by the router, together with the map's threshold, sets participation. This suggests that learned scales or shapes are the right knob for adaptive MoE routing.

\section{Conclusions}
\label{sec:conclusion}

We studied whether the sparsity of a probability map survives training in a top-2 MoE router.
Across matched 300M and 1B models, the answer is that it survives only partly, and never as a property of the map alone.
The $1.5$-entmax router retains more probability mass than softmax while keeping both selected experts, the sparsemax router retains the most mass, and the normmax-2 router uses the fewest contributors.
These behaviors are set jointly by the map's sparsification threshold and by the score distribution the router learns under it.
The $1.5$-entmax router keeps both experts because its learned scores have half the spread of softmax's, which keeps its top-2 gaps below the threshold, whereas sparsemax and normmax-2, which sparsify at the same gap, differ only because their routers learned different gap distributions.
Retention, participation and predictive quality also turn out to be three distinct quantities: none of the sparse maps improves validation loss over softmax, yet all of them make the trained model less sensitive than softmax to increasing the inference-time budget.

Our results also say where a sparse router should be designed, rather than merely swapped in.
Because the router co-adapts its score scale to the map, that scale is a modelling choice that we currently leave to the optimizer, and it would be better learned explicitly, as is already done for sparse attention~\citep{correia-etal-2019-adaptively,vasylenko2026longcontext}.
A router whose scale is under control could then be trusted to choose its own budget: the sparse maps already assign zero weight to most of the extra slots when $K$ is raised, so selecting every expert in the support instead of a fixed top-$K$, capped by an at-most-$k$ operator, is a small step from what the trained models already do.
The balancing objective is another design choice to test: its gradients depend on the map, so changing the probabilities used for balancing may alter the score scales and expert loads learned during training.
We leave these directions to future work.

\paragraph{Limitations.}
Our matched comparison uses a single training seed per configuration, so the small quality differences between maps (at most $0.012$ nats of validation loss) do not establish a ranking.
We test two members of the entmax family ($\alpha=1.5$ and sparsemax, $\alpha=2$), one normmax ($\gamma=2$) and one balancing objective, whose probabilities we vary in a single 300M ablation.
Contributor counts measure expert participation, not compute. As such, whether zero-weight experts are skipped, and what that saves, depends on the dispatch implementation, which we do not measure.
We did not verify that the diagnostic samples are disjoint from the training data.
Finally, the score-scale analysis is descriptive: it relates each router's singleton rate to its learned gap distribution but does not identify the training mechanism that produced the smaller scales.

\label{page:main-end}
\section*{Acknowledgments}
\label{app:acks}
We thank the SARDINE Lab members for reviewing this paper and providing helpful feedback.

This work was partially supported by Portuguese national funds through FCT under project UIDB/50021/2020.
Additionally, the work was partially supported by the project Center for Responsible AI, reference C628696807-00454142. 

\section*{AI Use Statement}
In this work, we used generative AI tools to assist in writing the proofs and fixing notation issues.
We also used generative AI tools for summarizing and analyzing existing literature, searching for information, discovering literature beyond our manual search, and editing and polishing the paper for readability.
We used generative AI tools to write and edit software code.
All uses of AI tools were accompanied by a careful review process. %
In particular, we manually checked the proofs and made them more accessible, and LLM-generated code was verified and tested for correctness.
We take responsibility for the final content of this work, including any text, claims, or artifacts produced with the aid of generative AI.

\section*{Reproducibility statement}
Appendix~\ref{app:empirical} gives the architecture shapes and parameter counts (\ref{app:architecture}), the full training recipe including optimizer settings, schedule, batch composition and the balancing objective (\ref{app:training}), and the evaluation protocol, including the quality metrics, the exact definition of every routing diagnostic and the inference-time budget experiment (\ref{app:evaluation}); Appendix~\ref{app:thresholds} derives the singleton thresholds and defines the score-scale analysis, and Appendix~\ref{app:extra-results} collects the per-layer and 300M results.
In the original matched comparison, all models at a given scale share the same data, seed and hyperparameters and differ only in the probability map. 
We will release the training and evaluation code, the configuration files of every run, the recorded router scores of every trained model used in \S\ref{sec:map-or-scores}, and a numerical supplement with the routing, score-scale and quality summaries at both scales, so that every table and figure in the paper can be regenerated.

\bibliographystyle{plainnat}
\bibliography{references,references_added}
\clearpage
\appendix

\section{Experimental details}
\label{app:empirical}

\FloatBarrier
\subsection{Architecture and parameter accounting}
\label{app:architecture}
All models use an untied input embedding and language-model head with vocabulary size $128{,}256$, pre-RMSNorm blocks and a final RMSNorm, grouped-query attention, rotary embeddings with base $500{,}000$, and bias-free SwiGLU feed-forward blocks~\citep{dubey2024llama3}.
The configured maximum context is $8{,}192$; training and the routing diagnostics use $2{,}048$ tokens.
The norm epsilon is $10^{-5}$, dropout is zero and the initialization standard deviation is $0.02$.
Table~\ref{tab:architectures} gives the shapes and parameter counts.

\begin{table}[!tbp]
\caption{Architectures and parameter counts. FF width is the hidden width of each SwiGLU branch; MoE layers replace every second feed-forward block. Nominal active parameters include the full embedding tensors and count two experts per MoE layer, even when one combine weight is zero.}
\label{tab:architectures}
\centering\small
\setlength{\tabcolsep}{4.5pt}
\begin{tabular}{lrrrrrrrrr}
\toprule
& & & Heads & FF & MoE & & \multicolumn{3}{c}{Parameters} \\
\cmidrule(lr){8-10}
Model & Layers & Width & Q/KV & width & layers & Experts & Total & Active & Router \\
\midrule
Dense $283$M & 10 & 768  & 12/3 & 3072 & --  & --  & 283M  & 283M  & --  \\
MoE $300$M   & 8  & 320  & 8/2  & 2048 & 4   & 26  & 297M  & 108M  & 33K \\
Dense $1$B   & 18 & 1536 & 24/6 & 6144 & --  & --  & 1.01B & 1.01B & --  \\
MoE $1$B     & 9  & 768  & 12/3 & 3072 & 4   & 27  & 1.01B & 302M  & 83K \\
\bottomrule
\end{tabular}
\end{table}

With $L$ layers, $L_{\mathrm M}$ MoE layers, width $d$, SwiGLU width $h$, vocabulary $V$ and $d_{\mathrm{KV}}=d\,n_{\mathrm{KV}}/n_{\mathrm Q}$, the parameter count is
\begin{align*}
P_{\mathrm{total}}={}&2Vd+L(2d^2+2dd_{\mathrm{KV}}+2d)+d+3dh[(L-L_{\mathrm M})+L_{\mathrm M}n]+L_{\mathrm M}dn,
\end{align*}
where the last term counts the router matrices.
Replacing $n$ by $K=2$ in the expert term gives the nominal active count, which includes the full embedding tensors and counts two experts per MoE layer even when one combine weight is zero.
The four maps add no parameters at fixed $\alpha$ and $\gamma$.
Nominal active parameters do not equate training or inference FLOPs across dense and MoE models.

\FloatBarrier
\subsection{Training recipe}
\label{app:training}
Training uses DCLM-baseline~\citep{li2024datacomp} with the Meta-Llama-3 tokenizer and seed $42$.
AdamW has peak learning rate $3\times10^{-4}$, minimum rate $3\times10^{-6}$, betas $(0.9,0.95)$, weight decay $0.1$ and gradient clipping at norm $1.0$; the learning rate warms up linearly for $500$ updates and then follows cosine decay to the minimum at update $19{,}073$.
Mixed precision is bfloat16.
Every run uses four GPUs with batch size $4$ per process and $16$ gradient-accumulation microbatches, i.e.\ $256$ sequences or $524{,}288$ tokens per optimizer update, and $19{,}073$ updates for the nominal $10$B-token budget.
The MoE runs use the Switch-style balancing loss of Equation~\eqref{eq:loads}, evaluated on each local microbatch of $8{,}192$ tokens with coefficient $0.01$ and averaged over the four MoE layers; there is no router z-loss, and the dense references have no balancing loss.
Validation cross-entropy on a $10$M-token cache is logged every $500$ updates and at the final update; for all ten models the minimum is at the final update, which is the checkpoint we evaluate.

\FloatBarrier
\subsection{Evaluation protocol}
\label{app:evaluation}
\label{app:quality-metrics}

\paragraph{Quality metrics.}
Validation loss and the diagnostic-domain losses exclude the balancing loss; every unmasked $2{,}048$-token chunk contributes $2{,}047$ next-token targets.
The eight OLMES tasks are ARC-Challenge, ARC-Easy, CommonsenseQA, HellaSwag, OpenBookQA, PIQA, SocialIQA and WinoGrande~\citep{gu2025olmes}, each evaluated in the multiple-choice and completion formulations with the OLMES scoring rules. We report the higher of the two accuracies per model and task, as recommended in \citet[\S3.4]{gu2025olmes}.
LAMBADA~\citep{paperno2016lambada} counts a prediction as correct only if every token of the final word is predicted correctly, and its word perplexity uses the full target-word log probability.
The mean accuracy weights the eight OLMES tasks and LAMBADA equally.
Evaluation runs in FP32 with autocast and TF32 disabled.
Table~\ref{tab:quality-primary} in the main text gives the per-task results at both scales.

\paragraph{Diagnostic samples.}
Domain losses use $489$ packed $2{,}048$-token blocks of retained DCLM ($1{,}000{,}983$ scored targets) and the $61$ Paloma WikiText-103 test documents ($285{,}107$ targets).
The inference-budget evaluation uses the retained DCLM sample.
All Paloma diagnostics reported here use the test sample at context $2{,}048$.

\paragraph{Routing diagnostics.}
\label{app:metrics}
\label{app:checkpoint-diagnostic}
For every MoE layer and token batch, the nominal load assigns $1/K$ to each selected index, fillers included, and the contributor load only to selections with strictly positive weight; histograms are divided by the number of routed tokens, and the CV is the population standard deviation divided by the mean, averaged over layers.
The router computes scores and probabilities in FP32 under mixed precision and casts the normalized combine weights $p_{ij}/m_i$ to the expert-output dtype.
The training curves of Figures~\ref{fig:discarded} and~\ref{fig:discarded-300m} use every logged training batch at a cadence of $10$ updates.
The checkpoint diagnostic runs every trained model in FP32 on the $61$ Paloma test documents with identical inputs, contexts and target order, giving $285{,}107$ targets per layer and $1{,}140{,}428$ target--layer events per model; the recorded router scores of all four layers are the data from which the spreads and gap quantiles of \S\ref{sec:map-or-scores} and Appendix~\ref{app:score-scale} are computed.
Targets and layers are dependent observations from one checkpoint per map, not training replications.

\paragraph{Inference-time budget.}
For the experiment of \S\ref{sec:budget} we take each trained model with frozen weights and run it with $K\in\{1,2,4,8\}$ selected experts on the retained DCLM diagnostic sample, keeping the map and renormalizing the selected probabilities as in Equation~\eqref{eq:combine}.
Besides the loss, we record the retained mass $1-\delta_i$ and the number of contributors $K_{+,i}$ of every target--layer event and average them over targets and layers; Table~\ref{tab:frozen-k} in the main text reports the loss and Table~\ref{tab:sparsity-transfer} in Appendix~\ref{app:extra-results} the full set.

\FloatBarrier
\section{Singleton thresholds and the score-scale analysis}
\label{app:thresholds}
This appendix derives the singleton thresholds quoted in \S\ref{sec:maps} from the optimality conditions of each map, records the closed forms that follow from them, and then defines the score-scale quantities used in \S\ref{sec:map-or-scores}.
Throughout, the token index is suppressed, $\bm z\in\R^n$ is a score vector with sorted entries $z_{(1)}\geq z_{(2)}\geq\dots\geq z_{(n)}$, $g=z_{(1)}-z_{(2)}$ is its top-2 gap, and $\mathcal I=\{j:p_j>0\}$ is the support of $\bm p=\pi(\bm z)$, with size $s=|\mathcal I|$.

\subsection{Optimality conditions}
\label{app:kkt}
Every map we use solves the regularized argmax problem of Equation~\eqref{eq:regularized},
\begin{equation}
\bm p=\argmax_{\bm p\in\simplex}\ \bm p^\top\bm z+H(\bm p),
\label{eq:app-regularized}
\end{equation}
for a concave entropy $H$.
Introducing a multiplier $\theta\in\R$ for the constraint $\bm 1^\top\bm p=1$ and multipliers $\mu_j\geq0$ for the constraints $p_j\geq0$, the Lagrangian is
\begin{equation}
\mathcal L(\bm p,\theta,\bm\mu)=\bm p^\top\bm z+H(\bm p)-\theta\,(\bm 1^\top\bm p-1)+\bm\mu^\top\bm p,
\end{equation}
and the Karush--Kuhn--Tucker (KKT) conditions read, for every expert $j$,
\begin{equation}
z_j+\partial_jH(\bm p)-\theta+\mu_j=0,\qquad \mu_j\geq0,\qquad \mu_jp_j=0,\qquad \bm p\in\simplex .
\label{eq:kkt}
\end{equation}
Because the objective is concave and the simplex is convex, these conditions are necessary and sufficient, and for the three entropies below the maximizer is unique~\citep{blondel2020learning,santos2025hopfield}.
Splitting Equation~\eqref{eq:kkt} by whether $p_j$ is positive gives the two conditions we use repeatedly:
\begin{equation}
z_j+\partial_jH(\bm p)=\theta\quad\text{for } j\in\mathcal I,
\qquad\qquad
z_j+\partial_jH(\bm p)\leq\theta\quad\text{for } j\notin\mathcal I .
\label{eq:kkt-split}
\end{equation}
The first says that on the support the score, shifted by the entropy gradient, is constant; the second says that an expert is left out exactly when its shifted score does not reach that constant.

\paragraph{Softmax.}
With the Shannon entropy $H^{\mathrm S}(\bm p)=-\sum_jp_j\log p_j$ we have $\partial_jH^{\mathrm S}(\bm p)=-\log p_j-1$, which diverges to $+\infty$ as $p_j\to0$.
The inactive condition in Equation~\eqref{eq:kkt-split} can therefore never hold, so every expert is on the support, and the active condition gives $\log p_j=z_j-1-\theta$, i.e.\ $p_j=e^{z_j}/\sum_ke^{z_k}$ after choosing $\theta$ so that the probabilities sum to one.

\paragraph{$\alpha$-entmax.}
The Tsallis entropy $H^{\mathrm{Ts}}_\alpha(\bm p)=\big(1-\sum_jp_j^\alpha\big)/(\alpha(\alpha-1))$ has gradient
\begin{equation}
\partial_jH^{\mathrm{Ts}}_\alpha(\bm p)=-\frac{\alpha p_j^{\alpha-1}}{\alpha(\alpha-1)}=-\frac{p_j^{\alpha-1}}{\alpha-1},
\end{equation}
which is finite at $p_j=0$ for every $\alpha>1$, so experts can leave the support.
Substituting in Equation~\eqref{eq:kkt-split}, the active condition becomes $z_j-p_j^{\alpha-1}/(\alpha-1)=\theta$, that is
\begin{equation}
p_j^{\alpha-1}=(\alpha-1)(z_j-\theta)
\quad\Longleftrightarrow\quad
p_j=\big[(\alpha-1)z_j-(\alpha-1)\theta\big]^{1/(\alpha-1)}\quad\text{for } j\in\mathcal I,
\end{equation}
and the inactive condition becomes $z_j\leq\theta$.
Both are captured by the single expression
\begin{equation}
p_j=\big[(\alpha-1)z_j-\tau\big]_+^{1/(\alpha-1)},\qquad \tau=(\alpha-1)\theta,
\label{eq:app-entmax}
\end{equation}
which is Equation~\eqref{eq:entmax}: the bracket is positive exactly for the experts with $z_j>\theta$, and those are the support.
The threshold $\tau$ is fixed by the simplex constraint,
\begin{equation}
\sum_j\big[(\alpha-1)z_j-\tau\big]_+^{1/(\alpha-1)}=1 .
\label{eq:app-entmax-norm}
\end{equation}
The left-hand side is continuous and strictly decreasing in $\tau$ wherever it is positive, tends to $+\infty$ as $\tau\to-\infty$ and vanishes for $\tau\geq(\alpha-1)z_{(1)}$, so Equation~\eqref{eq:app-entmax-norm} has exactly one root, which can always be found by bisection.
For the two values of $\alpha$ we use, the root also has a sorting-based closed form.
If the support is the top $k$ experts, then for sparsemax ($\alpha=2$) Equation~\eqref{eq:app-entmax-norm} reads $\sum_{j\leq k}(z_{(j)}-\tau)=1$, so
\begin{equation}
\tau_k=\frac{1}{k}\Big(\sum_{j\leq k}z_{(j)}-1\Big),
\label{eq:app-sparsemax-tau}
\end{equation}
and for $1.5$-entmax ($\alpha=3/2$, so that $(\alpha-1)=1/2$ and $1/(\alpha-1)=2$) it reads $\sum_{j\leq k}(z_{(j)}/2-\tau)^2=1$.
Writing $m_k$ and $v_k$ for the mean and the variance of the $k$ halved scores $z_{(1)}/2,\dots,z_{(k)}/2$, we have $\sum_{j\leq k}(z_{(j)}/2-\tau)^2=k\big[(m_k-\tau)^2+v_k\big]$, and the root that keeps the probabilities positive is
\begin{equation}
\tau_k=m_k-\sqrt{\frac{1-k\,v_k}{k}} .
\label{eq:app-entmax15-tau}
\end{equation}
In both cases the support size is the unique $k$ for which the candidate $\tau_k$ is consistent with Equation~\eqref{eq:kkt-split}, i.e.\ the $k$ largest bracket arguments are positive and the remaining ones are not; after one sort this is a linear scan over $k$~\citep{martins2016softmax,peters2019sparse}.
In our implementation $1.5$-entmax uses this exact algorithm, and sparsemax and normmax-2 solve their threshold equations by bisection.

\paragraph{$\gamma$-normmax.}
The norm entropy $H^{\mathrm N}_\gamma(\bm p)=1-\|\bm p\|_\gamma$ has gradient
\begin{equation}
\partial_jH^{\mathrm N}_\gamma(\bm p)=-\frac{p_j^{\gamma-1}}{\|\bm p\|_\gamma^{\gamma-1}}=-\Big(\frac{p_j}{R}\Big)^{\gamma-1},\qquad R=\|\bm p\|_\gamma,
\end{equation}
which again vanishes at $p_j=0$ for $\gamma>1$.
Equation~\eqref{eq:kkt-split} therefore gives $(p_j/R)^{\gamma-1}=z_j-\theta$ on the support and $z_j\leq\theta$ off it, so
\begin{equation}
p_j=R\,[z_j-\theta]_+^{1/(\gamma-1)} .
\label{eq:app-normmax-raw}
\end{equation}
Summing Equation~\eqref{eq:app-normmax-raw} over $j$ and using $\bm 1^\top\bm p=1$ identifies $R=1/\sum_k[z_k-\theta]_+^{1/(\gamma-1)}$, which turns Equation~\eqref{eq:app-normmax-raw} into the ratio form of Equation~\eqref{eq:normmax} with $\tau=\theta$.
Substituting Equation~\eqref{eq:app-normmax-raw} back into the definition $R^\gamma=\sum_jp_j^\gamma$ and dividing by $R^\gamma$ gives the normalization $\sum_j[z_j-\theta]_+^{\gamma/(\gamma-1)}=1$ stated in Equation~\eqref{eq:normmax}.
As for entmax, the left-hand side of this normalization is strictly decreasing in $\theta$, so the threshold is unique and can be found by bisection; for $\gamma=2$ it reads $\sum_j(z_j-\theta)_+^2=1$ and admits the same sorted closed form as Equation~\eqref{eq:app-entmax15-tau}, with the raw scores in place of the halved ones.

\paragraph{Order preservation.}
In each of the three cases, the active condition has the form $z_j-\theta=\phi(p_j)$ for a strictly increasing function $\phi$ with $\phi(0)=0$ (namely $\phi(p)=p^{\alpha-1}/(\alpha-1)$ for entmax and $\phi(p)=(p/R)^{\gamma-1}$ for normmax), and the inactive condition is $z_j-\theta\leq0$.
Hence $p_j$ is a nondecreasing function of $z_j$: larger scores receive larger probabilities, equal scores receive equal probabilities, and every expert with $p_j>0$ has a larger score than every expert with $p_j=0$.
Consequently the top-$K$ experts by probability are the top-$K$ by score (ties broken by the same rule), which is why the selection $\mathcal S$ of Equation~\eqref{eq:combine} and all the statements below can be phrased in terms of scores alone.

\subsection{Singleton thresholds}
\label{app:singleton}
\begin{proposition}[Singleton thresholds]
\label{prop:thresholds}
Let $g=z_{(1)}-z_{(2)}$.
The support of $\alpha$-entmax is the singleton $\{(1)\}$ if and only if $g\geq1/(\alpha-1)$; the support of $\gamma$-normmax is $\{(1)\}$ if and only if $g\geq1$, for every $\gamma>1$; and the support of softmax is never a singleton.
In particular the thresholds are $2$ for $1.5$-entmax and $1$ for both sparsemax and normmax-2.
\end{proposition}
\begin{proof}
For softmax the claim was shown above: all probabilities are positive at finite scores.
For the other two maps we test the candidate $\bm p=\bm e_{(1)}$ against the KKT conditions~\eqref{eq:kkt-split}; since the conditions are sufficient, the candidate is the solution whenever it satisfies them, and since the solution is unique, it is not the solution whenever it violates them.

\emph{Entmax.}
At $\bm p=\bm e_{(1)}$ the active condition for expert $(1)$ reads $1^{\alpha-1}=(\alpha-1)(z_{(1)}-\theta)$, which fixes the multiplier,
\begin{equation}
\theta=z_{(1)}-\frac{1}{\alpha-1}.
\end{equation}
The inactive condition for every other expert $j$ is $z_j\leq\theta$, that is $z_{(1)}-z_j\geq1/(\alpha-1)$.
Because $z_{(2)}$ is the largest of the remaining scores, all these inequalities hold if and only if the one for $j=(2)$ holds, i.e.\ if and only if $g\geq1/(\alpha-1)$.
At $\alpha=3/2$ this is $g\geq2$ and at $\alpha=2$ it is $g\geq1$.

\emph{Normmax.}
At $\bm p=\bm e_{(1)}$ the norm is $R=\|\bm e_{(1)}\|_\gamma=1$, so the active condition for expert $(1)$ reads $(1/1)^{\gamma-1}=z_{(1)}-\theta$, i.e.\ $\theta=z_{(1)}-1$ regardless of $\gamma$.
The inactive conditions $z_j\leq\theta$ for $j\neq(1)$ again reduce to the one for $(2)$, which is $g\geq1$.
\end{proof}

Two consequences of Proposition~\ref{prop:thresholds} are used in the main text.
Firstly, since the maps preserve score order, top-2 selection always picks $\{(1),(2)\}$, and the second selection is a filler exactly when $g$ reaches the threshold: for a trained router, the singleton rate reported in Tables~\ref{tab:headline} and~\ref{tab:score-gaps} is therefore the fraction of its top-2 gaps at or above its map's threshold, $\Pr(K_+{=}1)=\Pr(g\geq t_\pi)$ with $t_\pi\in\{2,1,1\}$ for $1.5$-entmax, sparsemax and normmax-2.
Secondly, sparsemax and normmax-2 share the threshold $1$, so on the same score vector they always agree on whether the token has one or two contributors, even though their full supports and their mixture weights differ.
Any difference between their contributor counts in Table~\ref{tab:headline} must therefore come from the scores their routers learned, which is the observation that \S\ref{sec:map-or-scores} develops.

\paragraph{How the second expert fades.}
The thresholds describe where the second selected expert reaches zero weight; the closed forms on a two-expert support show how it gets there.
Suppose only experts $(1)$ and $(2)$ are active, so that $p_{(1)}+p_{(2)}=1$, and subtract the two active conditions of Equation~\eqref{eq:kkt-split}.
For sparsemax, $p_{(1)}-p_{(2)}=g$, hence
\begin{equation}
p_{(2)}=\frac{1-g}{2},
\end{equation}
a weight that decreases linearly in the gap and vanishes at $g=1$.
For $1.5$-entmax, $\sqrt{p_{(1)}}-\sqrt{p_{(2)}}=g/2$; writing $a=\sqrt{p_{(1)}}$ and $b=\sqrt{p_{(2)}}$ with $a^2+b^2=1$, squaring gives $1-2ab=g^2/4$, so $ab=(4-g^2)/8$ and $a+b=\sqrt{1+2ab}=\sqrt{8-g^2}/2$; then $b=\tfrac12[(a+b)-(a-b)]$ yields
\begin{equation}
p_{(2)}=\frac{\big(\sqrt{8-g^2}-g\big)^2}{16},
\end{equation}
which equals $1/2$ at $g=0$ and vanishes at $g=2$.
For normmax-2, $(p_{(1)}-p_{(2)})/R=g$ with $R=\sqrt{p_{(1)}^2+p_{(2)}^2}$; parameterizing $p_{(1)}=(1+x)/2$ and $p_{(2)}=(1-x)/2$ gives $R^2=(1+x^2)/2$ and hence $x^2=g^2(1+x^2)/2$, so
\begin{equation}
p_{(2)}=\frac12\Big(1-\frac{g}{\sqrt{2-g^2}}\Big),
\end{equation}
which vanishes at $g=1$.
At a gap of $g=0.5$, for instance, the second expert keeps a weight of $0.250$ under sparsemax, $0.311$ under normmax-2 and $0.326$ under $1.5$-entmax; at the same scores, the maps with the larger threshold, or the more uniform tendency, hold on to the second expert longer.
These formulas apply whenever the third score is below the map's threshold, which by Equation~\eqref{eq:kkt-split} is the condition for the support to have exactly two experts.

\subsection{Retained mass of sparsemax and normmax-2 on the same scores}
\label{app:ordering}
Although sparsemax and normmax-2 agree on the contributor count, they do not agree on the discarded mass.
The following argument shows that, on the same score vector, sparsemax always keeps at least as much top-2 mass as normmax-2.

Let $\theta$ be the normmax-2 threshold and define the truncated scores $u_j(t)=(z_j-t)_+$ for any $t\in\R$.
By Equation~\eqref{eq:normmax} with $\gamma=2$, the normmax-2 probabilities are $p_j^{\mathrm N}=u_j(\theta)/\|\bm u(\theta)\|_1$ with $\|\bm u(\theta)\|_2=1$, and by Equation~\eqref{eq:app-entmax} with $\alpha=2$ the sparsemax probabilities are $p_j^{\mathrm S}=u_j(\tau)$ with $\|\bm u(\tau)\|_1=1$.
The argument has three steps.
\begin{enumerate}[leftmargin=1.5em,itemsep=1pt]
\item \emph{The sparsemax threshold is at least the normmax-2 threshold.}
For every vector, $\|\bm u\|_1\geq\|\bm u\|_2$, so $\|\bm u(\theta)\|_1\geq\|\bm u(\theta)\|_2=1=\|\bm u(\tau)\|_1$.
Since $t\mapsto\|\bm u(t)\|_1$ is nonincreasing, this forces $\tau\geq\theta$.
In particular the sparsemax support is contained in the normmax-2 support.
\item \emph{The top-2 share of truncated mass is nondecreasing in the threshold.}
Let $m(t)=\big(u_{(1)}(t)+u_{(2)}(t)\big)/\|\bm u(t)\|_1$ be the fraction of truncated mass on the two largest coordinates, and let $s(t)$ be the number of positive coordinates of $\bm u(t)$.
On any interval of $t$ where $s(t)\geq2$ is constant, the numerator has derivative $-2$ and the denominator has derivative $-s(t)$, so
\begin{equation}
\frac{\mathrm d}{\mathrm dt}\,m(t)=\frac{-2\,\|\bm u(t)\|_1+s(t)\big(u_{(1)}(t)+u_{(2)}(t)\big)}{\|\bm u(t)\|_1^2}\geq0,
\end{equation}
because the mean of the two largest positive coordinates is at least the mean of all $s(t)$ positive coordinates, i.e.\ $\big(u_{(1)}+u_{(2)}\big)/2\geq\|\bm u\|_1/s$.
The function $m$ is continuous where the support changes and equals $1$ once $s(t)\leq2$, so it is nondecreasing on all of $\R$.
\item \emph{Conclusion.}
The retained top-2 mass of normmax-2 is $m(\theta)$ and that of sparsemax is $u_{(1)}(\tau)+u_{(2)}(\tau)=m(\tau)$, because $\|\bm u(\tau)\|_1=1$.
By the two previous steps, $m(\tau)\geq m(\theta)$, that is $\delta^{\mathrm S}\leq\delta^{\mathrm N}$.
\end{enumerate}
This is a same-score statement.
The trained sparsemax and normmax-2 routers see different scores, so the ordering of their discarded masses in Table~\ref{tab:headline} ($0.153$ against $0.250$ at 1B) is consistent with the argument but not implied by it.

\subsection{Score scale and top-2 gaps}
\label{app:score-scale}
Proposition~\ref{prop:thresholds} turns the question of how often a router drops its second expert into a question about the distribution of its top-2 gaps.
The gap distribution has a scale and a shape, and the two play different roles.

\paragraph{Definitions.}
For a trained router and a fixed MoE layer, let $z_{i,e}$ be the recorded score of expert $e$ for target $i$, with $M=285{,}107$ targets per layer (Appendix~\ref{app:metrics}) and $n$ experts.
We center each target's scores across experts, $c_{i,e}=z_{i,e}-\bar z_i$ with $\bar z_i=\frac1n\sum_ec_{i,e}$, which removes the additive freedom that no map is sensitive to, and define the layer's score scale as the root mean square of the centered scores,
\begin{equation}
\mathrm{RMS}=\Big(\frac{1}{nM}\sum_{i=1}^{M}\sum_{e=1}^{n}c_{i,e}^2\Big)^{1/2},
\label{eq:app-rms}
\end{equation}
The top-2 gap of target $i$ is $g_i=z_{i,(1)}-z_{i,(2)}$, and $g_{q}$ denotes the $q$th percentile of $\{g_i\}$ within the layer.
Table~\ref{tab:score-gaps} in Appendix~\ref{app:score-distributions} averages the scale and the gap percentiles over the four MoE layers; Table~\ref{tab:score-scale-quantiles} in Appendix~\ref{app:extra-results} lists them per layer.
All these quantities are computed in float64 from the recorded FP32 scores, without re-running any model or map.

\paragraph{Why the scale matters.}
Multiplying the scores by $\rho>0$ multiplies every gap by $\rho$ and leaves the shape of the gap distribution unchanged, so by Proposition~\ref{prop:thresholds} the singleton rate of a map with threshold $t$ on the scaled scores is $\Pr(\rho\,g\geq t)$.
Equivalently, in terms of the standardized gap $G_i=g_i/\mathrm{RMS}$, the singleton rate is $\Pr(G\geq t/\mathrm{RMS})$: a router can change how often it drops an expert either by reshaping its standardized gaps or by moving its scale relative to the threshold.
Because the router weights are trained by ordinary backpropagation under the map, nothing pins the scale, which is why \S\ref{sec:map-or-scores} treats it as a learned quantity.

\paragraph{Gap fractions under each threshold.}
Since the thresholds are conditions on the scores alone, we can ask, for every trained router, what fraction of its gaps lies at or above $1$ (the sparsemax and normmax-2 threshold) and at or above $2$ (the $1.5$-entmax threshold); Table~\ref{tab:all-source-thresholds} reports both, next to the router's own singleton rate.
For each sparse router the own rate coincides with the fraction at its own threshold, as Proposition~\ref{prop:thresholds} requires, which is also a check of the recorded scores against the recorded routing.
The table makes the scale effect visible: at 1B, $7.8\%$ of the softmax router's gaps exceed $2$ and $44.8\%$ exceed $1$, whereas the $1.5$-entmax router, with an RMS less than half as large, has only $0.05\%$ of its gaps above $2$ and $8.3\%$ above $1$.
Sparsemax and normmax-2 sit at the same threshold, so the difference between their singleton rates ($6.9\%$ against $21.3\%$) is entirely a difference between the gap distributions their routers learned; Table~\ref{tab:score-gaps} shows that it comes from the whole distribution being shifted up rather than from a heavier tail, since normmax-2's median gap is $0.61$ against $0.40$ for sparsemax and its $90$th percentile $1.06$ against $0.81$.
These fractions describe the recorded score distributions in threshold units; they do not apply one router's map to another router's scores.
\begin{table}[!tbp]
\caption{Fraction of top-2 score gaps at least $1$ (the sparsemax and normmax-2 threshold) and at least $2$ (the $1.5$-entmax threshold) for every trained router, averaged over layers, next to the router's own singleton rate. For each sparse router the own rate equals the fraction at its own threshold; softmax has no finite threshold.}
\label{tab:all-source-thresholds}
\centering\small
\setlength{\tabcolsep}{5pt}
\begin{tabular}{l ccc ccc}
\toprule
& \multicolumn{3}{c}{300M} & \multicolumn{3}{c}{1B} \\
\cmidrule(lr){2-4}\cmidrule(lr){5-7}
\bf Map & $g\geq1$ (\%) & $g\geq2$ (\%) & Own (\%) & $g\geq1$ (\%) & $g\geq2$ (\%) & Own (\%) \\
\midrule
Softmax      & 45.702 & 13.785 & 0.000  & 44.757 & 7.824 & 0.000 \\
$1.5$-entmax & 15.347 & 0.295  & 0.295  & 8.258  & 0.048 & 0.048 \\
Sparsemax    & 9.598  & 0.001  & 9.598  & 6.883  & 0.006 & 6.883 \\
Normmax-$2$  & 23.144 & 0.133  & 23.144 & 21.256 & 0.014 & 21.256 \\
\bottomrule
\end{tabular}
\end{table}

\FloatBarrier
\section{Additional results}
\label{app:extra-results}

\FloatBarrier
\subsection{Routing statistics by layer}
\label{app:observed-support-loads}
Table~\ref{tab:layerwise} breaks the layer averages of Table~\ref{tab:headline} down by MoE layer.
The singleton rate is strongly depth-dependent for sparsemax and normmax-2, which route between a fifth (sparsemax) and over half (normmax-2) of the tokens to a single expert in the first MoE layer and under a tenth in the last, while $1.5$-entmax has essentially no singletons at any depth.
The two load conventions also separate by layer: writing $f_j^{\rm nom}=f_j^++f_j^0$ for the contributor and filler parts of the nominal load,
$\operatorname{Var}_j(f_j^{\rm nom})=\operatorname{Var}_j(f_j^+)+\operatorname{Var}_j(f_j^0)+2\operatorname{Cov}_j(f_j^+,f_j^0)$,
so a negative covariance between contributor and filler loads can mask contributor imbalance in the nominal histogram, which is what happens to normmax-2 in the first layers.
\begin{table}[!tbp]
\caption{Routing statistics of every MoE layer (network depths 2, 4, 6 and 8) on the Paloma WikiText-103 test documents. Each block reports one statistic of Table~\ref{tab:headline} per layer instead of its layer average.}
\label{tab:layerwise}
\centering\small
\setlength{\tabcolsep}{4pt}
\begin{tabular}{l cccc cccc}
\toprule
& \multicolumn{4}{c}{300M, MoE layer} & \multicolumn{4}{c}{1B, MoE layer} \\
\cmidrule(lr){2-5}\cmidrule(lr){6-9}
\bf Map & 2 & 4 & 6 & 8 & 2 & 4 & 6 & 8 \\
\midrule
\textcolor{gray}{\textit{Singleton rate $\Pr(s{=}1)$ (\%)}} \\
Softmax      & 0.000  & 0.000  & 0.000 & 0.000 & 0.000  & 0.000  & 0.000  & 0.000 \\
$1.5$-entmax & 0.059  & 0.048  & 0.072 & 1.000 & 0.003  & 0.001  & 0.031  & 0.159 \\
Sparsemax    & 27.934 & 5.451  & 2.002 & 3.003 & 19.267 & 0.887  & 3.573  & 3.806 \\
Normmax-$2$  & 58.281 & 18.621 & 9.147 & 6.526 & 54.455 & 11.980 & 11.483 & 7.104 \\
\midrule
\textcolor{gray}{\textit{Discarded mass $\bar\delta$}} \\
Softmax      & 0.609 & 0.510 & 0.556 & 0.588 & 0.682 & 0.671 & 0.626 & 0.664 \\
$1.5$-entmax & 0.332 & 0.267 & 0.279 & 0.369 & 0.448 & 0.480 & 0.423 & 0.496 \\
Sparsemax    & 0.077 & 0.068 & 0.095 & 0.159 & 0.114 & 0.210 & 0.133 & 0.156 \\
Normmax-$2$  & 0.117 & 0.102 & 0.144 & 0.220 & 0.133 & 0.307 & 0.227 & 0.333 \\
\midrule
\textcolor{gray}{\textit{Contributors $\overline K_+$}} \\
Softmax      & 2.000 & 2.000 & 2.000 & 2.000 & 2.000 & 2.000 & 2.000 & 2.000 \\
$1.5$-entmax & 1.999 & 2.000 & 1.999 & 1.990 & 2.000 & 2.000 & 2.000 & 1.998 \\
Sparsemax    & 1.721 & 1.946 & 1.980 & 1.970 & 1.807 & 1.991 & 1.964 & 1.962 \\
Normmax-$2$  & 1.417 & 1.814 & 1.909 & 1.935 & 1.455 & 1.880 & 1.885 & 1.929 \\
\midrule
\textcolor{gray}{\textit{Nominal load CV}} \\
Softmax      & 0.255 & 0.322 & 0.508 & 0.530 & 0.260 & 0.401 & 0.305 & 0.391 \\
$1.5$-entmax & 0.261 & 0.332 & 0.501 & 0.505 & 0.236 & 0.341 & 0.326 & 0.448 \\
Sparsemax    & 0.279 & 0.389 & 0.591 & 0.479 & 0.367 & 0.444 & 0.419 & 0.426 \\
Normmax-$2$  & 0.256 & 0.343 & 0.473 & 0.396 & 0.244 & 0.388 & 0.372 & 0.361 \\
\midrule
\textcolor{gray}{\textit{Contributor load CV}} \\
Softmax      & 0.255 & 0.322 & 0.508 & 0.530 & 0.260 & 0.401 & 0.305 & 0.391 \\
$1.5$-entmax & 0.261 & 0.332 & 0.501 & 0.506 & 0.236 & 0.341 & 0.326 & 0.449 \\
Sparsemax    & 0.330 & 0.389 & 0.591 & 0.491 & 0.396 & 0.444 & 0.417 & 0.437 \\
Normmax-$2$  & 0.355 & 0.356 & 0.498 & 0.409 & 0.373 & 0.414 & 0.386 & 0.377 \\
\bottomrule
\end{tabular}
\end{table}

\subsection{Score distributions and scale by layer}
\label{app:score-distributions}
Table~\ref{tab:score-gaps} gives the score spreads and gap quantiles behind Figure~\ref{fig:score-gaps}, together with the corresponding 300M values. Its quantiles are averaged across layers, whereas the histograms pool the equally sized layer samples.
\begin{table}[t]
\caption{Learned score spread and top-2 gap distribution of the trained routers on the Paloma WikiText-103 test documents. RMS is the root mean square of the centered scores and $g_{50}$, $g_{90}$, $g_{99}$ are percentiles of the top-2 score gap $z_{(1)}-z_{(2)}$, all averaged over the four MoE layers. A selected expert is dropped exactly when the gap reaches the map's threshold, so the singleton rate is the fraction of gaps at or above it.}
\label{tab:score-gaps}
\centering\small
\setlength{\tabcolsep}{3pt}
\begin{tabular}{lc ccccc ccccc}
\toprule
& & \multicolumn{5}{c}{300M} & \multicolumn{5}{c}{1B} \\
\cmidrule(lr){3-7} \cmidrule(lr){8-12}
\bf Map & Threshold & RMS & $g_{50}$ & $g_{90}$ & $g_{99}$ & Singletons (\%) & RMS & $g_{50}$ & $g_{90}$ & $g_{99}$ & Singletons (\%) \\
\midrule
Softmax      & --  & 0.98 & 0.91 & 2.14 & 3.15 & 0.00  & 0.86 & 0.90 & 1.88 & 2.92 & 0.00 \\
$1.5$-entmax & $2$ & 0.52 & 0.45 & 1.12 & 1.71 & 0.30  & 0.36 & 0.44 & 0.95 & 1.45 & 0.05 \\
Sparsemax    & $1$ & 0.41 & 0.36 & 0.84 & 1.27 & 9.60  & 0.32 & 0.40 & 0.81 & 1.20 & 6.88 \\
Normmax-$2$  & $1$ & 0.53 & 0.64 & 1.12 & 1.60 & 23.14 & 0.38 & 0.61 & 1.06 & 1.42 & 21.26 \\
\bottomrule
\end{tabular}
\end{table}

Table~\ref{tab:score-scale-quantiles} gives the per-layer values behind Table~\ref{tab:score-gaps}.
The ordering of the scales is the same in every layer: softmax is the largest and sparsemax the smallest, with $1.5$-entmax and normmax-2 close together in between.
The $1.5$-entmax router's $99$th-percentile gap stays below its threshold of $2$ in every layer, with the single exception of the last MoE layer at 300M, where it reaches $2.000$ and where that router also has its highest singleton rate ($1.0\%$, Table~\ref{tab:layerwise}).
For sparsemax and normmax-2, whose threshold is $1$, the same percentiles explain the depth dependence of the singleton rate: the first MoE layer is the only one where normmax-2's median gap exceeds $1$, and it is the layer where both routers drop the second expert most often.
\begin{table}[!tbp]
\caption{Score scale and top-2 gap percentiles of every trained router in every MoE layer (network depths 2, 4, 6 and 8) on the Paloma WikiText-103 test documents; Table~\ref{tab:score-gaps} averages these values over layers. RMS is the root mean square of the centered scores and $g_q$ the $q$th percentile of the top-2 gap.}
\label{tab:score-scale-quantiles}
\centering\small
\setlength{\tabcolsep}{4pt}
\begin{tabular}{l cccc cccc}
\toprule
& \multicolumn{4}{c}{300M, MoE layer} & \multicolumn{4}{c}{1B, MoE layer} \\
\cmidrule(lr){2-5}\cmidrule(lr){6-9}
\bf Map & 2 & 4 & 6 & 8 & 2 & 4 & 6 & 8 \\
\midrule
\textcolor{gray}{\textit{RMS}} \\
Softmax      & 0.875 & 1.082 & 1.037 & 0.935 & 0.818 & 0.847 & 0.861 & 0.902 \\
$1.5$-entmax & 0.467 & 0.531 & 0.552 & 0.518 & 0.355 & 0.354 & 0.379 & 0.352 \\
Sparsemax    & 0.384 & 0.398 & 0.455 & 0.416 & 0.322 & 0.283 & 0.336 & 0.339 \\
Normmax-$2$  & 0.500 & 0.539 & 0.554 & 0.532 & 0.406 & 0.348 & 0.374 & 0.373 \\
\midrule
\textcolor{gray}{\textit{$g_{50}$}} \\
Softmax      & 0.849 & 1.131 & 0.877 & 0.773 & 0.937 & 0.998 & 0.993 & 0.678 \\
$1.5$-entmax & 0.398 & 0.566 & 0.498 & 0.346 & 0.530 & 0.435 & 0.425 & 0.354 \\
Sparsemax    & 0.429 & 0.440 & 0.315 & 0.237 & 0.663 & 0.314 & 0.370 & 0.262 \\
Normmax-$2$  & 1.137 & 0.578 & 0.468 & 0.358 & 1.044 & 0.490 & 0.518 & 0.369 \\
\midrule
\textcolor{gray}{\textit{$g_{90}$}} \\
Softmax      & 2.430 & 2.294 & 1.865 & 1.968 & 1.924 & 1.813 & 1.967 & 1.799 \\
$1.5$-entmax & 1.345 & 1.132 & 1.094 & 0.907 & 1.042 & 0.897 & 0.930 & 0.916 \\
Sparsemax    & 1.148 & 0.873 & 0.711 & 0.614 & 1.063 & 0.689 & 0.773 & 0.730 \\
Normmax-$2$  & 1.415 & 1.181 & 0.981 & 0.902 & 1.282 & 1.033 & 1.030 & 0.892 \\
\midrule
\textcolor{gray}{\textit{$g_{99}$}} \\
Softmax      & 3.067 & 3.087 & 2.795 & 3.649 & 2.715 & 2.414 & 3.049 & 3.501 \\
$1.5$-entmax & 1.711 & 1.567 & 1.557 & 2.000 & 1.463 & 1.244 & 1.455 & 1.649 \\
Sparsemax    & 1.319 & 1.254 & 1.120 & 1.383 & 1.250 & 0.989 & 1.213 & 1.361 \\
Normmax-$2$  & 1.639 & 1.525 & 1.369 & 1.845 & 1.551 & 1.266 & 1.428 & 1.419 \\
\bottomrule
\end{tabular}
\end{table}

\subsection{Inference-time budget}
\label{app:frozen-k}
Table~\ref{tab:sparsity-transfer} completes Table~\ref{tab:frozen-k} with the retained mass and the number of contributors at every budget.
Every map has its lowest loss at the trained $K=2$, so retaining more mass by selecting more experts does not help, and the sparse maps convert a larger budget into far fewer contributors than softmax: at $K=8$ the 1B sparsemax router uses $4.3$ experts on average and normmax-2 $5.1$, against $8.0$ for softmax and $7.98$ for $1.5$-entmax.
\begin{table}[!tbp]
\caption{Frozen-weight inference at budgets $K\in\{1,2,4,8\}$ on the retained DCLM diagnostic sample ($489$ blocks, $1{,}000{,}983$ scored targets) for models trained at $K=2$: loss (nats per target), retained mass $1-\bar\delta$ and number of contributors $\overline K_+$.}
\label{tab:sparsity-transfer}
\centering\small
\setlength{\tabcolsep}{4pt}
\begin{tabular}{l cccc cccc}
\toprule
& \multicolumn{4}{c}{300M, budget $K$} & \multicolumn{4}{c}{1B, budget $K$} \\
\cmidrule(lr){2-5}\cmidrule(lr){6-9}
\bf Map & 1 & 2 & 4 & 8 & 1 & 2 & 4 & 8 \\
\midrule
\textcolor{gray}{\textit{Loss (nats)}} \\
Softmax      & 3.703 & 3.560 & 3.623 & 3.913 & 3.360 & 3.253 & 3.328 & 3.835 \\
$1.5$-entmax & 3.711 & 3.562 & 3.598 & 3.693 & 3.363 & 3.261 & 3.306 & 3.512 \\
Sparsemax    & 3.715 & 3.565 & 3.566 & 3.568 & 3.365 & 3.263 & 3.273 & 3.284 \\
Normmax-$2$  & 3.703 & 3.562 & 3.574 & 3.586 & 3.368 & 3.266 & 3.286 & 3.329 \\
\midrule
\textcolor{gray}{\textit{Retained mass $1-\bar\delta$}} \\
Softmax      & 0.323 & 0.437 & 0.544 & 0.688 & 0.255 & 0.340 & 0.434 & 0.582 \\
$1.5$-entmax & 0.502 & 0.694 & 0.833 & 0.951 & 0.406 & 0.543 & 0.664 & 0.821 \\
Sparsemax    & 0.669 & 0.914 & 0.990 & 1.000 & 0.637 & 0.848 & 0.966 & 0.999 \\
Normmax-$2$  & 0.662 & 0.867 & 0.968 & 0.997 & 0.617 & 0.766 & 0.880 & 0.973 \\
\midrule
\textcolor{gray}{\textit{Contributors $\overline K_+$}} \\
Softmax      & 1.000 & 2.000 & 4.000 & 8.000 & 1.000 & 2.000 & 4.000 & 8.000 \\
$1.5$-entmax & 1.000 & 1.998 & 3.987 & 7.854 & 1.000 & 1.999 & 3.996 & 7.978 \\
Sparsemax    & 1.000 & 1.896 & 2.796 & 3.064 & 1.000 & 1.933 & 3.369 & 4.296 \\
Normmax-$2$  & 1.000 & 1.771 & 2.781 & 3.451 & 1.000 & 1.760 & 3.091 & 5.077 \\
\bottomrule
\end{tabular}
\end{table}

\subsection{Full results at 300M}
\label{app:small-scale}
The 300M models use an eight-layer backbone with 26 experts in each of the four MoE layers (Table~\ref{tab:architectures}).
Their quality, routing and budget results are reported next to the 1B ones in Tables~\ref{tab:quality-primary}, \ref{tab:headline} and~\ref{tab:frozen-k} of the main text. Table~\ref{tab:score-gaps} gives their score-gap summaries, and Tables~\ref{tab:layerwise} and~\ref{tab:score-scale-quantiles} give their per-layer statistics. Figures~\ref{fig:discarded-300m} and~\ref{fig:score-gaps-300m} show their training curves and gap distributions.

The picture matches 1B.
Softmax has the lowest validation loss and the highest mean accuracy among the MoEs in these runs.
The $1.5$-entmax router retains more mass than softmax with almost two contributors ($99.3\%$ of its states have support larger than two, with conditional discarded mass $0.314$ against $0.566$ for softmax), sparsemax retains the most mass and normmax-2 uses the fewest contributors.
The sparse routers again learn smaller score spreads than softmax (RMS $0.41$--$0.53$ against $0.98$); the $1.5$-entmax router's $99$th-percentile gap is $1.71$, below its threshold of $2$, while $13.8\%$ of the softmax router's gaps exceed $2$, and normmax-2's gaps exceed $1$ far more often than sparsemax's ($23.1\%$ against $9.6\%$).
The one difference from 1B is the load balance: at 300M, normmax-2 has a lower nominal CV than softmax ($0.367$ against $0.404$) but the same contributor CV ($0.404$), and $1.5$-entmax has the lowest contributor CV ($0.400$).

\begin{figure}[tbp]
\centering
\includegraphics[width=\linewidth]{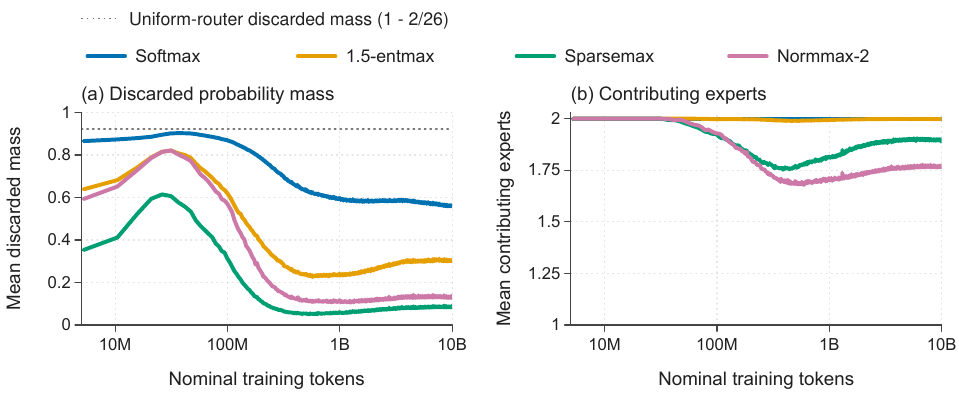}
\caption{Discarded probability mass (a) and number of contributing experts (b) during training of the 300M MoE models, measured on training batches; the horizontal axis is the number of training tokens (log scale).}
\label{fig:discarded-300m}
\end{figure}

\begin{figure}[tbp]
\centering
\includegraphics[width=\linewidth]{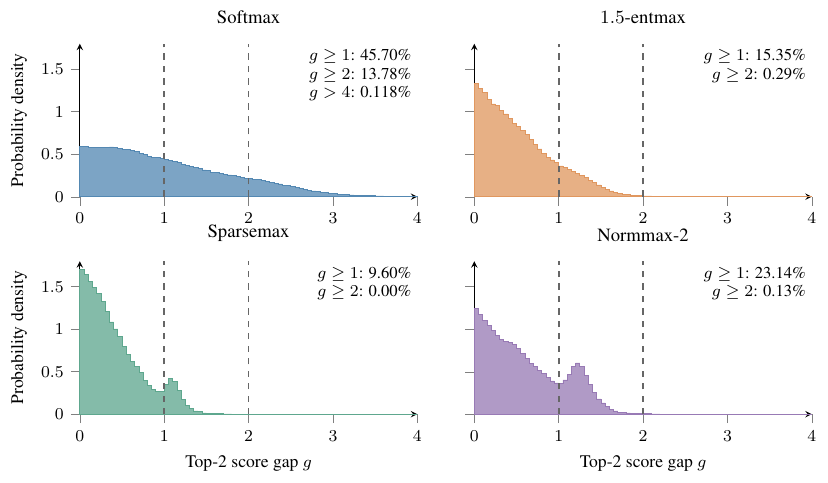}
\caption{Top-2 score gaps of the original 300M seed-42 routers on the same 61 Paloma WikiText-103 test documents. Samples, normalization and threshold lines follow Figure~\ref{fig:score-gaps}. The peaks near $g=1$ for sparsemax and normmax-2 are likewise consistent with gradient saturation beyond the singleton threshold.}
\label{fig:score-gaps-300m}
\end{figure}

\begin{figure}[tbp]
\centering
\includegraphics[width=\linewidth]{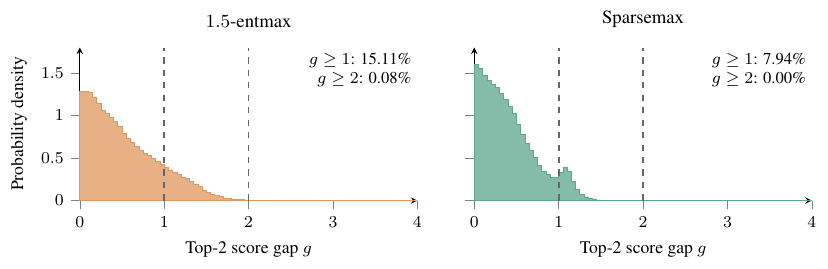}
\caption{Top-2 score gaps of the 300M sparsemax and $1.5$-entmax routers trained with the softmax-based balancing loss. Each panel pools $285{,}107$ target positions from each of the four MoE layers with equal weight, using the same 61 Paloma WikiText-103 test documents as Figure~\ref{fig:score-gaps-300m}. Dashed lines mark the singleton threshold $g=1$ for sparsemax and $g=2$ for $1.5$-entmax, with $g=1$ also shown in the entmax panel for reference.}
\label{fig:score-gaps-auxsoftmax}
\end{figure}

\end{document}